\documentclass[11pt]{article}
\usepackage[a4paper,margin=27mm]{geometry}
\usepackage{amsmath,amssymb,amsthm,mathtools,bm}
\usepackage{microtype}
\microtypesetup{expansion=false}
\usepackage{booktabs}
\usepackage{graphicx}
\usepackage{placeins}
\usepackage[hidelinks]{hyperref}
\usepackage[nameinlink,noabbrev]{cleveref}

\newtheorem{definition}{Definition}[section]
\newtheorem{proposition}[definition]{Proposition}
\newtheorem{theorem}[definition]{Theorem}
\newtheorem{corollary}[definition]{Corollary}
\newtheorem{remark}[definition]{Remark}

\newcommand{\hhat}[1]{\underline{\widehat{#1}}}
\newcommand{\hbreve}[1]{\underline{\breve{#1}}}
\newcommand{\dual}[1]{\underline{#1}}
\newcommand{\HU}{\hhat{U}}
\newcommand{\HSO}{\hhat{SO}_3}
\newcommand{\DQ}{\dual{U}}
\newcommand{\przero}{\operatorname{pr}_{0}}
\newcommand{\coef}[1]{[\varepsilon^{#1}]}
\newcommand{\Sclerp}{\operatorname{ScLERP}_{\!H}}
\newcommand{\Ad}{\operatorname{Ad}}
\newcommand{\Log}{\operatorname{Log}}

\title{Arbitrary-Order Hermite Interpolation of Rigid-Motion Jets\\
via Hyper-Multidual Quaternions}
\author{Daniel Condurache\\[0.6ex]
\small Technical University of Iasi, D. Mangeron Street no.~59, 700050 Iasi, Romania\\
\small Center for Industrial Robots Simulation and Testing, Technical University of Cluj-Napoca,\\
\small Memorandumului 14, 400114 Cluj-Napoca, Romania\\
\small \texttt{daniel.condurache@tuiasi.ro}}
\date{\today}

\begin{document}
\maketitle

\begin{abstract}
We study bilateral interpolation of finite-order rigid-motion jets represented
by unit dual quaternions.  An order-$n$ multidual (MD) algebra is the truncated
polynomial algebra $\mathbb R[\varepsilon]/(\varepsilon^{n+1})$;
hyper-multidual (HMD) quaternions are dual quaternions with coefficients in
this algebra.  Temporal HMD transforms encode a pose and its derivatives,
whereas a generic HMD curve need not be the temporal jet of its pose
projection; we call this requirement holonomicity.  We show that a temporal
transform and its relative descriptor are unitary and derive recursive
coefficient constraints, together with a local realizability converse in an
admissible logarithm chart.  We then extend screw linear interpolation
(ScLERP) algebraically to unit HMD quaternions.  Although it matches complete
endpoint transforms, direct HMD--ScLERP is generically non-holonomic for
arbitrary endpoint jets.  We give a coefficient criterion and explicit
endpoint and first-order interior contact defects.  A holonomic alternative
is obtained by mapping endpoint transforms to logarithmic dual-quaternion
coordinates, applying the degree-$(2n+1)$ Hermite polynomial that matches
derivatives through order $n$, and lifting by the exponential.  HMD arithmetic
also recovers higher-order rigid-motion acceleration fields without explicit
differentiation of $\mathrm{dexp}$.  Rotation and full $SE(3)$ tests through
second order, with an additional third-order polynomial check, reproduce the
stated defects and endpoint jets.
\end{abstract}

\section{Introduction}

Geometric interpolation of rotations and rigid motions has been developed in
several representations.  Park and Ravani construct invariant interpolation
schemes on the rotation group \cite{ParkRavani1997}; dual quaternions provide
a screw-theoretic representation of rigid displacement \cite{Selig2005}; and
the dual-quaternion power formula gives screw linear interpolation (ScLERP)
between two unit dual quaternions \cite{KavanEtAl2008}.  \v{Z}efran and Kumar study
rigid-motion interpolation with endpoint poses and prescribed endpoint twists
\cite{ZefranKumar1998}.  Higher derivatives of kinematic mappings and twists
are developed for lower-pair chains in \cite{Muller2014,Muller2019}.

The algebraic ingredients used here have distinct sources and scopes.
Messelmi develops multidual numbers and their associated multidual functions
\cite{Messelmi2015}, while hyper-dual arithmetic provides exact first- and
second-derivative evaluation \cite{FikeAlonso2012}.  The MD and HMD
differential transforms, orthogonal HMD tensors, unit HMD quaternions, and
higher-order rigid-body fields used in this paper are those introduced in
\cite{Condurache2025}.

The purpose of this paper is to answer a deceptively simple question: can two
unit HMD quaternions be interpolated by the usual ScLERP formula?  Algebraically
the answer is affirmative.  Kinematically, an additional issue appears.  A
generic curve in the HMD group need not be the differential transform of the
curve obtained by projecting it onto ordinary dual quaternions.  This is a
holonomicity, or contact, condition.  Confusing these two levels would identify a
generic $\hhat q$ with a differential transform $\hbreve q$ and would make
the endpoint interpolation statement stronger than justified.

The contributions of this paper are therefore:
\begin{enumerate}
  \item a direct ScLERP construction in the Lie group $\HU$ of unit HMD
  quaternions;
  \item an exact statement of what is interpolated when the endpoints are
  differential transforms;
  \item explicit unit-jet compatibility constraints obtained from the unit
  relative HMD descriptor $\hbreve q\dual q^*$;
  \item a necessary and sufficient coefficient criterion for holonomicity,
  endpoint defect formulas, and a generic non-holonomicity theorem;
  \item a degree-$(2n+1)$ HMD--Hermite correction that matches arbitrary
  compatible endpoint jets without explicit $\mathrm{dexp}$ differentiation;
  \item an algebraic recovery of higher-order acceleration fields from
  $\hbreve q\dual q^*$, requiring no explicit Lie-bracket expansion in the
  computational construction.
\end{enumerate}

The degree-$(2n{+}1)$ HMD--Hermite construction below matches arbitrary
bilateral endpoint jets, in finite algebraic form, at any prescribed finite
order $n\ge 1$.  Related constructions in the literature address important
lower-order cases or point-and-velocity data.  Kim, Kim, and Shin \cite{KimKimShin1995}
give Hermite quaternion curves in the rotation subgroup; Belta and Kumar
\cite{BeltaKumar2002} interpolate on $SE(3)$ with prescribed endpoint
velocities by SVD-based projection; logarithmic-coordinate Hermite
interpolation on Riemannian manifolds is developed for point and velocity
data in \cite{Zimmermann2020}; and a geometric construction producing $C^k$
curves from point and velocity data is given in \cite{RodriguesEtAl2005}.
Recent exponential-coordinate formulations on $SE(3)$ include
\cite{Muller2025}.  In the present construction, Lie brackets are absorbed
into noncommutative HMD products and are not generated term by term.

\section{Multidual differential algebra and hyper-multidual representations of rigid motion}

We follow the notation of \cite{Condurache2025}: a hat denotes an MD object,
an underline together with a hat denotes an HMD object, a breve denotes the MD
differential transform of a real time function, and an underlined breve denotes
the HMD differential transform of a dual-valued time function.  Thus
$\widehat x$ is MD, whereas $\hhat x$ is HMD; similarly, $\breve f$ is MD and
$\hbreve q$ is HMD.

\subsection{The multidual algebra}

Fix $n\geq1$ and let $\varepsilon$ be a nilpotent generator of order $n+1$.
The order-$n$ multidual algebra is the quotient
\begin{equation}
 \mathbb R_n=\mathbb R[\varepsilon]/\langle\varepsilon^{n+1}\rangle
 =\left\{a=\sum_{k=0}^{n}a_k\varepsilon^k:
 a_k\in\mathbb R\right\}.
 \label{eq:md-algebra}
\end{equation}
It is a commutative, associative, unital real algebra of dimension $n+1$.
For the algebraic theory of multidual numbers and the corresponding multidual
functions, see \cite{Messelmi2015}.
The ideal
\begin{equation}
 \mathfrak m=\varepsilon\mathbb R_n
 =\{a\in\mathbb R_n:a_0=0\}
 \label{eq:md-nilpotent-ideal}
\end{equation}
is nilpotent and $\mathbb R_n/\mathfrak m\simeq\mathbb R$.  Hence
$\mathbb R_n$ is a local algebra: $a$ is invertible if and only if its real
part $a_0$ is nonzero.

Multiplication is the truncated Cauchy product
\begin{equation}
 \left(\sum_{k=0}^na_k\varepsilon^k\right)
 \left(\sum_{k=0}^nb_k\varepsilon^k\right)
 =\sum_{k=0}^n\left(\sum_{j=0}^ka_jb_{k-j}\right)\varepsilon^k,
 \qquad
 \varepsilon^j\varepsilon^k=
 \begin{cases}\varepsilon^{j+k},&j+k\leq n,\\0,&j+k>n.
 \end{cases}
 \label{eq:md-product}
\end{equation}
The canonical projection $\pi_0:\mathbb R_n\to\mathbb R$ is
$\pi_0(a)=a_0$.

Let $f\in C^n(I,\mathbb R)$.  Its multidual prolongation is defined, for
$\widehat x=x+\Delta\widehat x$ with $x\in I$ and
$\Delta\widehat x\in\mathfrak m$, by the finite formula
\begin{equation}
 f(\widehat x)
 =\sum_{j=0}^{n}\frac{(\Delta\widehat x)^j}{j!}f^{(j)}(x).
 \label{eq:md-functional-calculus}
\end{equation}
No convergence assumption is involved because
$(\Delta\widehat x)^{n+1}=0$.

\begin{proposition}[Conservation of real functional identities]
\label{prop:md-identity-conservation}
Let $f,g,h\in C^n(I,\mathbb R)$ satisfy
\begin{equation}
 f(x)=g(x)h(x),\qquad x\in I.
 \label{eq:real-functional-identity}
\end{equation}
Then their multidual prolongations satisfy
\begin{equation}
 \boxed{f(\widehat x)=g(\widehat x)h(\widehat x)}
 \label{eq:md-identity-conservation}
\end{equation}
for every $\widehat x\in\mathbb R_n$ with $\pi_0(\widehat x)\in I$.
Addition, scalar multiplication, and admissible compositions of functional
identities are preserved analogously.
\end{proposition}

\begin{proof}
Differentiate \eqref{eq:real-functional-identity} through order $n$, use the
Leibniz formula, and substitute the resulting derivatives in
\eqref{eq:md-functional-calculus}.  The coefficients agree term by term with
the truncated Cauchy product \eqref{eq:md-product}.
\end{proof}

\subsection{The multidual differential transform}

Let $A$ be a finite-dimensional real associative algebra and
$f:I\to A$ a $C^n$ curve.  Its order-$n$ multidual differential transform is
\begin{equation}
 \mathcal T_n[f](t)=e^{\varepsilon D_t}f(t)
 =\sum_{k=0}^{n}\frac{\varepsilon^k}{k!}f^{(k)}(t)
 \in A\otimes_{\mathbb R}\mathbb R_n.
 \label{eq:differential-transform-general}
\end{equation}
Thus $\mathcal T_n[f](t)$ is the truncated Taylor shift $f(t+\varepsilon)$.
The factorials belong to the differential transform, not to the definition of
a generic multidual number.  Combined with the Cauchy product, they encode the
Leibniz rule.

\begin{proposition}[Properties of the MD transform]
For compatible $C^n$ functions of time $f=f(t)$ and $g=g(t)$, and a constant
$c$,
\begin{align}
 \mathcal T_n[f+g]&=\mathcal T_n[f]+\mathcal T_n[g],
 &\mathcal T_n[cf]&=c\mathcal T_n[f],\notag\\
 \mathcal T_n[fg]&=\mathcal T_n[f]\mathcal T_n[g],
 &\mathcal T_n[f^{-1}]&=\mathcal T_n[f]^{-1},
 \label{eq:md-transform-algebra}
\end{align}
where the inverse identity holds exactly when the zeroth coefficient $f(t)$ is
invertible in $A$ (equivalently, when $\mathcal T_n[f](t)$ is invertible).
If $f$ is scalar-valued, the substitution property below holds for every
$F\in C^n$ for which $F\circ f$ is defined. More generally, when $f$ takes
values in an associative algebra $A$, assume that $F$ is a polynomial or an
analytic function defined by a local functional calculus on a neighborhood
containing the range of $f$. Then
\begin{equation}
 \mathcal T_n[F\circ f](t)=F\!\left(\mathcal T_n[f](t)\right)
 \label{eq:md-transform-composition}
\end{equation}
holds, where the right-hand side is the MD prolongation of that same local
functional calculus. For multivalued functions, in particular $\Log$, one
admissible branch is fixed on the entire neighborhood under consideration.
Consequently, functional identities of time functions are transformed
according to Proposition~\ref{prop:md-identity-conservation}.
\end{proposition}

\begin{proof}
Linearity follows from \eqref{eq:differential-transform-general}.  By the
Leibniz formula,
$(fg)^{(k)}=\sum_{j=0}^k\binom{k}{j}f^{(j)}g^{(k-j)}$.  Since
$\binom{k}{j}/k!=1/(j!(k-j)!)$, this is exactly the coefficient obtained from
the Cauchy product \eqref{eq:md-product}.  Applying multiplicativity to
$ff^{-1}=1$ gives the inverse identity. For scalar-valued $f$, the composition
identity follows by applying the finite prolongation formula
\eqref{eq:md-functional-calculus} to the temporal jet
$\mathcal T_n[f](t)=f(t)+\Delta\widehat f(t)$ and using the ordinary chain
rule through order $n$. For algebra-valued $f$ and polynomial $F$, it follows
from linearity and multiplicativity applied to every monomial of $F$. The
analytic case follows from the same identity in the chosen local functional
calculus; fixing one branch makes the argument single-valued for $\Log$.
\end{proof}

\subsection{Dualization of the multidual algebra}

Introduce a new nilpotent generator $\varepsilon_0$, independent of and
commuting with $\varepsilon$, with $\varepsilon_0^2=0$.  The dualization of
$\mathbb R_n$ is
\begin{equation}
 \hhat{\mathbb R}_n
 =\mathbb R_n[\varepsilon_0]/\langle\varepsilon_0^2\rangle
 =\mathbb R_n\oplus\varepsilon_0\mathbb R_n
 \simeq\mathbb D\otimes_{\mathbb R}\mathbb R_n,
 \label{eq:hmd-dualization}
\end{equation}
where $\mathbb D=\mathbb R[\varepsilon_0]/\langle\varepsilon_0^2\rangle$ is
the ordinary dual-number algebra.  An HMD number therefore has either of the
equivalent forms
\begin{equation}
 \hhat a=\widehat a+\varepsilon_0\widehat a_0,
 \quad \widehat a,\widehat a_0\in\mathbb R_n,
 \qquad
 \hhat a=\sum_{k=0}^n
 (a_k+\varepsilon_0a_{0k})\varepsilon^k.
 \label{eq:hmd-number}
\end{equation}
Its product is obtained by dualization:
\begin{equation}
 (\widehat a+\varepsilon_0\widehat a_0)
 (\widehat b+\varepsilon_0\widehat b_0)
 =\widehat a\widehat b
 +\varepsilon_0(\widehat a\widehat b_0+\widehat a_0\widehat b),
 \label{eq:hmd-dual-product}
\end{equation}
where every product on the right is the multidual product
\eqref{eq:md-product}.  Thus the two nilpotencies have different roles:
$\varepsilon$ records derivatives through order $n$, whereas
$\varepsilon_0$ performs the dualization used to represent rigid displacement.

The MD transform extends coefficientwise to a dual-valued curve
$\alpha(t)=a(t)+\varepsilon_0a_0(t)$:
\begin{equation}
 \hbreve\alpha(t)=\mathcal T_n[\alpha](t)
 =\mathcal T_n[a](t)+\varepsilon_0\mathcal T_n[a_0](t)
 =\sum_{k=0}^n\frac{\varepsilon^k}{k!}\alpha^{(k)}(t).
 \label{eq:hmd-differential-transform}
\end{equation}
All identities in \eqref{eq:md-transform-algebra} remain valid after
dualization.

More generally, let $f,g,h$ be $C^n$ dual functions satisfying the dual
functional identity
\begin{equation}
 f(\dual x)=g(\dual x)h(\dual x).
 \label{eq:dual-functional-identity}
\end{equation}
For an HMD argument $\hhat x=\dual x+\Delta\hhat x$, with dual part
$\dual x\in\mathbb D$ and multidual part
$\Delta\hhat x\in\varepsilon\hhat{\mathbb R}_n$, define
\begin{equation}
 f(\hhat x)
 =\sum_{j=0}^n\frac{(\Delta\hhat x)^j}{j!}f^{(j)}(\dual x).
 \label{eq:hmd-function-prolongation}
\end{equation}
The dualized form of Proposition~\ref{prop:md-identity-conservation} gives
\begin{equation}
 \boxed{f(\hhat x)=g(\hhat x)h(\hhat x).}
 \label{eq:hmd-identity-conservation}
\end{equation}
Thus identities of real functions are conserved by MD prolongation, and
identities of dual functions are conserved by HMD prolongation.  The breve
accent is not used for these prolongations; it is reserved below for
differential transforms of functions of time.

\subsection{HMD tensors and the orthogonal HMD group}

Let $\hhat V_3=(\hhat{\mathbb R}_n)^3$ and denote by
$\mathcal L(\hhat V_3)$ the algebra of HMD linear tensors.  Transposition,
determinant, and analytic matrix functions are obtained by extending their
ordinary algebraic formulas over the commutative coefficient ring
$\hhat{\mathbb R}_n$.  The proper orthogonal HMD tensors form
\begin{equation}
 \HSO=\{\hhat R\in\mathcal L(\hhat V_3):
 \hhat R\hhat R^T=I,\ \det\hhat R=1\}.
 \label{eq:hmd-so3}
\end{equation}

Let $\dual{SO}_3$ denote the group of proper orthogonal dual tensors used
to represent rigid displacements.  For a $C^n$ motion
$\dual R:I\to \dual{SO}_3$, define
\begin{equation}
 \hbreve R(t)=\mathcal T_n[\dual R](t)
 =e^{\varepsilon D_t}\dual R(t)
 =\sum_{k=0}^n\frac{\varepsilon^k}{k!}\dual R^{(k)}(t).
 \label{eq:tensor-differential-transform}
\end{equation}

\begin{proposition}[Preservation of HMD orthogonality]
\label{prop:hmd-orthogonal-preservation}
For every $C^n$ curve $\dual R(t)\in \dual{SO}_3$,
\begin{equation}
 \hbreve R\hbreve R^T=I,
 \qquad \det\hbreve R=1.
 \label{eq:hmd-tensor-orthogonality}
\end{equation}
Consequently, $\hbreve R(t)\in\HSO$.
\end{proposition}

\begin{proof}
Apply the multiplicative differential transform to $\dual R\dual R^T=I$.  Since
transposition commutes with differentiation,
$\mathcal T_n[\dual R^T]=\hbreve R^T$, which proves the first identity.  Applying the
functional property to the determinant gives
$\det\hbreve R=\mathcal T_n[\det\dual R]=1$.
\end{proof}

The product in \eqref{eq:hmd-tensor-orthogonality} must be distinguished from
the tensor that identifies the higher-order kinematic fields.  With
$\dual R(t)=\przero\hbreve R(t)$ evaluated at the same physical time, define
\begin{equation}
 \boxed{\hhat\Psi=\hbreve R\dual R^T.}
 \label{eq:tensor-field-descriptor}
\end{equation}
Here only the first factor is differentially transformed.  The HMD orthogonal
tensor $\hhat\Psi$ admits the decomposition
\begin{equation}
 \hhat\Psi
 =\left(I+\varepsilon_0\widehat a^{\times}\right)\widehat\Phi,
 \label{eq:tensor-field-decomposition}
\end{equation}
where $v^{\times}w=v\times w$, $\widehat\Phi$ is the multidual tensorial part,
and $\widehat a$ is the multidual vectorial part.  They have expansions
\begin{equation}
 \widehat\Phi=I+\sum_{k=1}^n\frac{\varepsilon^k}{k!}\Phi_k,
 \qquad
 \widehat a=\sum_{k=1}^n\frac{\varepsilon^k}{k!}a_k.
 \label{eq:tensor-field-expansions}
\end{equation}
For a point $r$ fixed in the body, the complete higher-order field is
\begin{equation}
 \boxed{\widehat a_r=\widehat a+\widehat\Phi r.}
 \label{eq:tensor-point-field}
\end{equation}
Thus the coefficients of $\widehat a_r$ simultaneously give velocity,
acceleration, jerk, and all orders retained by the algebra
\cite{Condurache2025}.

\begin{remark}[Two distinct tensor products]
The identities
\begin{equation}
 \hbreve R\hbreve R^T=I,
 \qquad
 \hbreve R\dual R^T=\hhat\Psi
 \label{eq:two-tensor-products}
\end{equation}
have different meanings.  The first expresses membership in $\HSO$; the
second is the relative HMD tensor that carries the higher-order kinematic
fields.
\end{remark}

\subsection{HMD quaternions, unit quaternions, and the tensor covering}

Let $\hhat V_3=(\hhat{\mathbb R}_n)^3$.  The HMD quaternion algebra is
\begin{equation}
 \hhat{\mathbb Q}_n
 =\hhat{\mathbb R}_n\oplus\hhat V_3,
 \qquad
 \hhat q=(\hhat q_s,\hhat{\bm q}).
 \label{eq:hmd-quaternion}
\end{equation}
For $\hhat p=(\hhat p_s,\hhat{\bm p})$ and
$\hhat q=(\hhat q_s,\hhat{\bm q})$, define
\begin{equation}
 \hhat p\hhat q
 =\left(
  \hhat p_s\hhat q_s-\hhat{\bm p}\cdot\hhat{\bm q},
  \hhat p_s\hhat{\bm q}+\hhat q_s\hhat{\bm p}
  +\hhat{\bm p}\times\hhat{\bm q}
 \right).
 \label{eq:hmd-quaternion-product}
\end{equation}
The algebra is associative and noncommutative.  Quaternion conjugation acts on
the vector part only:
\begin{equation}
 \hhat q^*=(\hhat q_s,-\hhat{\bm q}),
 \qquad |\hhat q|^2=\hhat q\hhat q^*.
 \label{eq:hmd-conjugate}
\end{equation}
In particular, both $\varepsilon$ and $\varepsilon_0$ are unchanged by
quaternion conjugation.

The unit HMD quaternions form the Lie group
\begin{equation}
 \HU=\{\hhat q\in\hhat{\mathbb Q}_n:
             \hhat q\hhat q^*=1\}.
 \label{eq:hmd-unit-group}
\end{equation}
Its Lie algebra is identified with the pure-vector HMD quaternions.  Locally,
\begin{equation}
 \hhat q=\exp\!\left(\frac12\hhat\alpha\hhat u\right)
 =\cos\frac{\hhat\alpha}{2}
  +\hhat u\sin\frac{\hhat\alpha}{2},
 \qquad \hhat u\cdot\hhat u=1.
 \label{eq:hmd-polar}
\end{equation}
The exponential and a locally selected logarithm are evaluated by the finite
functional calculus inherited first from $\mathbb R_n$ and then by
dualization \cite{Condurache2025}.

\begin{definition}[Admissible logarithm chart]
An admissible logarithm chart is an open set $\dual{\mathcal U}_{\Log}\subset\DQ$
on which a single smooth branch
$\Log:\dual{\mathcal U}_{\Log}\to\dual{\mathfrak u}$ has been fixed.  Quaternion
representatives are first chosen consistently modulo the sign ambiguity, and
every curve to which $\Log$ is applied is required to remain in
$\dual{\mathcal U}_{\Log}$.  In particular, the rotation cut locus (including the
half-turn ambiguity) is excluded from the selected chart.  The Lie algebra
$\dual{\mathfrak u}$ consists of pure dual quaternions: both the real and dual
parts of any $\dual\xi\in\dual{\mathfrak u}$ have zero scalar coefficient.
\label{def:log-chart}
\end{definition}

The quaternion and tensor descriptions are related by the surjective Lie-group
homomorphism
\begin{equation}
 \Theta:\HU\longrightarrow\HSO,
 \qquad
 \Theta(\hhat q)
 =I+2\hhat q_s\hhat{\bm q}^{\times}
   +2(\hhat{\bm q}^{\times})^2.
 \label{eq:hmd-covering-map}
\end{equation}
It satisfies
\begin{equation}
 \Theta(\hhat p\hhat q)
 =\Theta(\hhat p)\Theta(\hhat q),
 \qquad
 \Theta(\hhat q^*)=\Theta(\hhat q)^T.
 \label{eq:hmd-covering-properties}
\end{equation}

Let $\DQ$ denote the ordinary unit dual-quaternion group.  For a $C^n$ curve
$\dual q:I\to\DQ$, its HMD differential transform is
\begin{equation}
 \hbreve q(t)=\mathcal T_n[\dual q](t)
 =e^{\varepsilon D_t}\dual q(t)
 =\sum_{k=0}^{n}\frac{\varepsilon^k}{k!}\dual q^{(k)}(t).
 \label{eq:differential-transform}
\end{equation}
The transform respects quaternion multiplication, conjugation, inverse, and
finite functional prolongation.  If $\dual q(t)$ remains in one admissible logarithm
chart, then
\begin{equation}
 \mathcal T_n[\dual q^*]=\mathcal T_n[\dual q]^*,
 \quad
 \mathcal T_n[\exp\dual X]=\exp(\mathcal T_n[\dual X]),
 \quad
 \mathcal T_n[\Log\dual q]=\Log(\mathcal T_n[\dual q]).
 \label{eq:transform-exp-log}
\end{equation}
If $\dual q(t)\dual q(t)^*=1$, multiplicativity gives
\begin{equation}
 \hbreve q(t)\hbreve q(t)^*=\mathcal T_n[\dual q\dual q^*](t)=1,
 \label{eq:transform-unitarity}
\end{equation}
and hence $\hbreve q(t)\in\HU$.

\begin{proposition}[Unit temporal transform and unit relative descriptor]
\label{prop:unit-hmd-jet-constraints}
Let $\dual q:I\to\DQ$ be $C^n$ and write
\begin{equation}
 \hbreve q
 =\sum_{r=0}^{n}\frac{\varepsilon^r}{r!}\dual q^{(r)}.
 \label{eq:unit-jet-expansion}
\end{equation}
Both $\hbreve q$ and the relative HMD quaternion
\begin{equation}
 \hhat\varphi
 :=\hbreve q\dual q^*
 =\sum_{r=0}^{n}\frac{\varepsilon^r}{r!}\dual\varphi_r,
 \qquad
 \dual\varphi_r:=\dual q^{(r)}\dual q^*,
 \quad \dual\varphi_0=1,
 \label{eq:relative-descriptor-expansion}
\end{equation}
are unitary.  Consequently, their relative coefficients satisfy
\begin{equation}
 \boxed{
 \sum_{k=0}^{r}\binom{r}{k}
 \dual\varphi_k\dual\varphi_{r-k}^{*}=0,
 \qquad r=1,\ldots,n.}
 \label{eq:relative-descriptor-constraints}
\end{equation}
Equivalently, for $r\geq1$,
\begin{equation}
 \boxed{
 \dual\varphi_r+\dual\varphi_r^*
 =-\sum_{k=1}^{r-1}\binom{r}{k}
 \dual\varphi_k\dual\varphi_{r-k}^{*}.}
 \label{eq:relative-descriptor-recursion}
\end{equation}
Thus the conjugate-symmetric part of $\dual\varphi_r$, and hence the
corresponding constrained part of $\dual q^{(r)}=\dual\varphi_r\dual q$, is
determined by lower-order derivatives.
Conversely, let $\hhat p\in\HU$ have base (coefficient independent of
$\varepsilon$) $\dual p_0$, and suppose that
$\dual p_0^*\hhat p$ belongs to the HMD prolongation of an admissible
logarithm chart.  Then $\hhat p$ is locally the order-$n$ temporal transform
of a smooth unit dual-quaternion curve through $\dual p_0$.
\end{proposition}

\begin{proof}
The first unitarity identity is
$\hbreve q\hbreve q^*=\mathcal T_n[\dual q\dual q^*]=1$.  Since
$\dual q^*\dual q=1$, the second follows from
\begin{equation}
 \hhat\varphi\hhat\varphi^*
 =\hbreve q\dual q^*\dual q\hbreve q^*
 =\hbreve q\hbreve q^*=1.
 \label{eq:relative-descriptor-unitarity}
\end{equation}
Extracting the coefficient of $\varepsilon^r/r!$ in this identity gives
\eqref{eq:relative-descriptor-constraints}.  Separating its terms with
$k=0$ and $k=r$, and using $\dual\varphi_0=1$, gives
\eqref{eq:relative-descriptor-recursion}.

For the converse, write the finite HMD logarithm as
\begin{equation}
 \hhat X=\Log(\dual p_0^*\hhat p)
 =\sum_{r=1}^{n}\frac{\varepsilon^r}{r!}\dual X_r.
 \label{eq:unit-jet-local-logarithm}
\end{equation}
Since $\dual p_0^*\hhat p$ is unitary and lies in the prolongation of the
chosen chart, $\hhat X$ belongs to the prolonged pure dual-quaternion Lie
algebra.  Define
\begin{equation}
 \dual p(\tau)=\dual p_0\exp\!\left(
 \sum_{r=1}^{n}\frac{\tau^r}{r!}\dual X_r\right).
 \label{eq:unit-jet-local-realization}
\end{equation}
This is a smooth unit dual-quaternion curve.  Functional prolongation and the
finite Taylor identity give
$\mathcal T_n[\dual p](0)=\dual p_0\exp(\hhat X)=\hhat p$, proving local
realizability.
\end{proof}

The first three restrictions are therefore
\begin{align}
 \dual\varphi_1+\dual\varphi_1^*&=0,
 \label{eq:relative-constraint-order1}\\
 \dual\varphi_2+\dual\varphi_2^*
   +2\dual\varphi_1\dual\varphi_1^*&=0,
 \label{eq:relative-constraint-order2}\\
 \dual\varphi_3+\dual\varphi_3^*
   +3\dual\varphi_1\dual\varphi_2^*
   +3\dual\varphi_2\dual\varphi_1^*&=0.
 \label{eq:relative-constraint-order3}
\end{align}
In particular, $\dual\varphi_1=\dot{\dual q}\dual q^*$ is pure.  The
higher coefficients $\dual\varphi_r=\dual q^{(r)}\dual q^*$ are not generally
pure; their conjugate-symmetric parts are fixed recursively by
\eqref{eq:relative-descriptor-recursion}.  These endpoint restrictions are
calculable directly from the unit descriptor, but unitarity alone should not
be confused with interior holonomicity of a generic HMD interpolating curve.

If $\dual R(t)=\Theta(\dual q(t))$, functional prolongation and the homomorphism property imply
\begin{equation}
 \hbreve R=\Theta(\hbreve q).
 \label{eq:transform-covering-compatibility}
\end{equation}
Moreover, the relative quaternion $\hhat\varphi$ defined in
\eqref{eq:relative-descriptor-expansion} is the quaternionic field descriptor,
and one obtains the structural identity
\begin{equation}
 \boxed{
 \Theta(\hhat\varphi)
 =\Theta(\hbreve q\dual q^*)
 =\hbreve R\dual R^T
 =\hhat\Psi.}
 \label{eq:descriptor-commuting-identity}
\end{equation}
Hence the quaternionic and tensorial descriptors encode the same hierarchy of
higher-order kinematic fields.

\begin{remark}[Spatial and material trivializations]
The descriptor used throughout this paper is the spatially trivialized product
$\hbreve q\dual q^*$.  The product $\dual q^*\hbreve q$ gives the corresponding material
trivialization; the two are related by the adjoint action of $\dual q$.  Neither
should be replaced by $\dual q\hbreve q^*$, which is a different product.
\end{remark}

Finally, we reserve the underlined breve for elements known to be HMD
differential transforms; an underlined hat denotes an arbitrary HMD element:
\begin{equation}
  \hbreve q=e^{\varepsilon D_t}\dual q\quad\Longrightarrow\quad
  \hbreve q\in\HU,
  \qquad
  \hhat q\in\HU\quad\not\Longrightarrow\quad
  \hhat q=e^{\varepsilon D_t}\dual q.
  \label{eq:hat-breve}
\end{equation}
The projection $\przero$ onto the coefficient independent of $\varepsilon$
satisfies $\przero\hbreve q=\dual q$, but a generic $\hhat q$ need not be a
holonomic jet.

\section{HMD--ScLERP}

All relative quaternions below are assumed to lie in one admissible chart from
Definition~\ref{def:log-chart}.  This includes a continuous choice of unit
quaternion representatives before the logarithm is evaluated.

\begin{definition}[Algebraic HMD--ScLERP]
For $\hhat q_0,\hhat q_1\in\HU$ lying in a common logarithm chart, define
\begin{equation}
 \Sclerp(s;\hhat q_0,\hhat q_1)
 =\hhat q_0\exp\!\left(s\Log(\hhat q_0^*\hhat q_1)\right)
 =\hhat q_0(\hhat q_0^*\hhat q_1)^s,
 \qquad 0\leq s\leq 1.
 \label{eq:hmd-sclerp}
\end{equation}
\end{definition}

\begin{proposition}[Group interpolation]
The curve $\hhat Q(s)=\Sclerp(s;\hhat q_0,\hhat q_1)$ lies in $\HU$ and
satisfies
\begin{equation}
 \hhat Q(0)=\hhat q_0,
 \qquad
 \hhat Q(1)=\hhat q_1.
 \label{eq:endpoints}
\end{equation}
It is left-equivariant: for every constant $\hhat p\in\HU$,
\begin{equation}
 \Sclerp(s;\hhat p\hhat q_0,\hhat p\hhat q_1)
 =\hhat p\,\Sclerp(s;\hhat q_0,\hhat q_1).
\end{equation}
\end{proposition}

\begin{proof}
The relative element $\hhat q_0^*\hhat q_1$ belongs to $\HU$.
Exponentiating its logarithm produces a one-parameter subgroup in $\HU$.
Equation \eqref{eq:endpoints} follows by setting $s=0$ and $s=1$.
Left equivariance follows because the relative element is unchanged:
$(\hhat p\hhat q_0)^*(\hhat p\hhat q_1)
=\hhat q_0^*\hhat q_1$.
\end{proof}

\begin{corollary}[Exact algebraic interpolation of endpoint transforms]
Let $\dual q=\dual q(t)$ be a time-dependent unit dual quaternion and let
\begin{equation}
 \hbreve q_i=\left.e^{\varepsilon D_t}\dual q(t)\right|_{t=t_i},
 \qquad i=0,1,
 \label{eq:temporal-endpoint-transforms}
\end{equation}
be two prescribed order-$n$ temporal differential transforms.  Then
\begin{equation}
 \hhat Q(s)=\Sclerp(s;\hbreve q_0,\hbreve q_1)
 \label{eq:transform-endpoint-sclerp}
\end{equation}
interpolates both HMD transforms exactly.  In particular, all endpoint
coefficients encoded in $\hbreve q_0$ and $\hbreve q_1$ are matched.
\end{corollary}

\begin{remark}
The interpolant in \eqref{eq:transform-endpoint-sclerp} is deliberately denoted
$\hhat Q$, not $\hbreve Q$.  The latter notation becomes legitimate only
after the holonomicity condition in the next section has been established.
\end{remark}

\section{Holonomicity and contact defects}

Let $t\in[t_0,t_1]$ be physical time, set
\begin{equation}
 T=t_1-t_0,
 \qquad s(t)=\frac{t-t_0}{T},
 \qquad
 \dual Q(t)=\przero\hhat Q(s(t)).
 \label{eq:pose-projection}
\end{equation}
There are two independent roles in the construction: $\varepsilon$ stores jet
coefficients with respect to physical time $t$, while $s$ is only the
normalized interpolation variable.  Therefore membership in
$\HU$ alone does not identify the $\varepsilon$ coefficients with successive
$t$-derivatives of $\dual Q$.

\begin{definition}[Holonomic HMD curve]
A curve $\hhat Q:[0,1]\to\HU$ is order-$n$ temporally holonomic along
$s=s(t)$ if
\begin{equation}
 \hhat Q(s(t))=\hbreve Q(t)=e^{\varepsilon D_t}\dual Q(t),
 \qquad \dual Q(t)=\przero\hhat Q(s(t)).
 \label{eq:holonomicity}
\end{equation}
The breve notation is used only for the differential transform of the time
function $\dual Q=\dual Q(t)$.  This is the finite-jet contact condition specialized to
the present HMD representation \cite{Saunders1989}.
\end{definition}

\begin{theorem}[Coefficient characterization]
A $C^n$ curve $\hhat Q:[0,1]\to\HU$ is temporally holonomic if and only if,
for every
$k=1,\ldots,n$,
\begin{equation}
 k!\,\coef{k}\hhat Q(s(t))
 =\frac{d^k\dual Q}{dt^k}(t)
 \qquad\text{for all }t\in[t_0,t_1].
 \label{eq:coefficient-test}
\end{equation}
\end{theorem}

\begin{proof}
Expanding the right-hand side of \eqref{eq:holonomicity} in the truncated algebra
gives
$\sum_{k=0}^n\varepsilon^k \dual Q^{(k)}(t)/k!$.  Equality of two HMD quaternions is
equivalent to equality of their coefficients.  This proves necessity and
sufficiency.
\end{proof}

\begin{definition}[Contact defects]
For a generic HMD interpolation curve define its temporal contact defects by
\begin{equation}
 \mathcal C_k[\hhat Q](t)
 =k!\,\coef{k}\hhat Q(s(t))
 -\frac{d^k\dual Q}{dt^k}(t),
 \qquad k=1,\ldots,n.
 \label{eq:contact-defects}
\end{equation}
Then $\hhat Q$ is temporally holonomic if and only if
$\mathcal C_k[\hhat Q]\equiv0$ for all $k$.
\end{definition}

Applied to \eqref{eq:transform-endpoint-sclerp}, this yields the contact
calculation
\begin{equation}
 \mathcal C_k\!\left[
 \hbreve q_0\exp\{s\Log(\hbreve q_0^*\hbreve q_1)\}
 \right](t),
 \qquad k=1,\ldots,n.
 \label{eq:central-defect}
\end{equation}
Let
\begin{equation}
 \dual\xi=\Log(\dual q_0^*\dual q_1),
 \qquad
 \dual Q(t)=\dual q_0\exp(s(t)\dual\xi)
 \label{eq:base-generator}
\end{equation}
be the pose projection of the direct HMD--ScLERP curve.  Since $\dual\xi$ is
constant,
\begin{equation}
 \frac{d^k\dual Q}{dt^k}(t)=\dual Q(t)\left(\frac{\dual\xi}{T}\right)^k.
 \label{eq:sclerp-derivatives}
\end{equation}

\begin{theorem}[Endpoint defects and generic non-holonomicity]
\label{thm:endpoint-defects-non-holonomicity}
Suppose
$\hbreve q_i=\sum_{k=0}^n\varepsilon^k \dual q_i^{(k)}/k!$, $i=0,1$, are prescribed
differential transforms.  For the direct interpolant
\begin{equation*}
 \hhat Q(s)=\Sclerp(s;\hbreve q_0,\hbreve q_1),
\end{equation*}
the endpoint contact defects are
\begin{align}
 \mathcal C_k[\hhat Q](t_0)&=\dual q_0^{(k)}
 -\dual q_0\left(\frac{\dual\xi}{T}\right)^k,
 \label{eq:left-endpoint-defect}\\
 \mathcal C_k[\hhat Q](t_1)&=\dual q_1^{(k)}
 -\dual q_1\left(\frac{\dual\xi}{T}\right)^k,
 \label{eq:right-endpoint-defect}
\end{align}
for $k=1,\ldots,n$.  Hence direct HMD--ScLERP is not holonomic for generic
endpoint jets.  Already at first order, holonomicity requires
\begin{equation}
 \dual q_0^{(1)}=\dual q_0\frac{\dual\xi}{T},
 \qquad \dual q_1^{(1)}=\dual q_1\frac{\dual\xi}{T}.
 \label{eq:first-order-compatibility}
\end{equation}
\end{theorem}

\begin{proof}
At $t=t_i$, the coefficient part of the HMD endpoint is prescribed by
$k!\coef{k}\hhat Q(s(t_i))=\dual q_i^{(k)}$.  On the other hand,
\eqref{eq:sclerp-derivatives} gives
$\dual Q^{(k)}(t_i)=\dual q_i(\dual\xi/T)^k$.  Substitution in
\eqref{eq:contact-defects} proves
\eqref{eq:left-endpoint-defect}--\eqref{eq:right-endpoint-defect}.  Conditions
\eqref{eq:first-order-compatibility} define a proper subset of possible
endpoint velocities, which proves generic non-holonomicity.
\end{proof}

\begin{corollary}[Constant-screw case]\label{cor:constant-screw}
If the two endpoint transforms are restrictions at $t_0,t_1$ of the same
constant-screw motion
$\dual q(t)=\dual q_0\exp(s(t)\dual\xi)$, then direct HMD--ScLERP is temporally holonomic.
\end{corollary}

\begin{proof}
In the commuting subalgebra generated by $\dual\xi$,
$s(t)$ is affine and therefore $s(t+\varepsilon)=s(t)+\varepsilon/T$.  Hence
$\hbreve q(t)=\dual q_0\exp((s(t)+\varepsilon/T)\dual\xi)$.  Consequently the relative HMD
element between the endpoints has logarithm $\dual\xi$, and
$\hhat Q(s(t))=\dual q_0\exp((s(t)+\varepsilon/T)\dual\xi)
=e^{\varepsilon D_t}\dual Q(t)$.
\end{proof}

\begin{proposition}[First-order interior defect]
Let $n=1$, write $\dual v_i=\dual q_i^{(1)}$, set
\begin{equation}
 \dual a=\dual q_0^*\dual q_1,\qquad
 \dual{\delta a}=\dual v_0^*\dual q_1+\dual q_0^*\dual v_1,\qquad
 \dual\eta=\mathrm D\Log_{\dual a}[\dual{\delta a}].
 \label{eq:first-order-relative-data}
\end{equation}
For the direct HMD--ScLERP, the complete first-order contact defect is
\begin{equation}
 \boxed{\mathcal C_1[\hhat Q](t)
 =\dual v_0e^{s\dual\xi}+\dual q_0\,\mathrm D\exp_{s\dual\xi}[s\dual\eta]
 -\dual q_0e^{s\dual\xi}\frac{\dual\xi}{T},\qquad s=s(t).}
 \label{eq:first-order-interior-defect}
\end{equation}
Here $\mathrm D\Log$ and $\mathrm D\exp$ denote Fr\'echet differentials of
the local mutually inverse maps
$\Log:\dual{\mathcal U}_{\Log}\subset\DQ\to\dual{\mathfrak u}$ and
$\exp:\dual{\mathfrak u}\to\DQ$.  More precisely, if
$\dual a=\exp\dual\xi$, then
\begin{equation*}
 \mathrm D\Log_{\dual a}:T_{\dual a}\DQ\longrightarrow\dual{\mathfrak u},
 \qquad
 \mathrm D\exp_{\dual\xi}:\dual{\mathfrak u}\longrightarrow T_{\dual a}\DQ,
 \qquad
 \mathrm D\Log_{\dual a}=(\mathrm D\exp_{\dual\xi})^{-1}.
\end{equation*}
Equivalently, an implementation may obtain $\dual\eta$ without assembling
either differential explicitly: it is the coefficient of $\varepsilon$ in
$\Log(\dual a+\varepsilon\dual{\delta a})$, evaluated by the same finite HMD
functional calculus used throughout the construction.
\end{proposition}

\begin{proof}
The first-order prolongations give
$\hbreve q_0^*\hbreve q_1=\dual a+\varepsilon\dual{\delta a}$ and hence
$\Log(\hbreve q_0^*\hbreve q_1)=\dual\xi+\varepsilon\dual\eta$.  Applying the defining
property of the Fr\'echet differential to the exponential yields
\begin{equation*}
 [\varepsilon]\hhat Q(s)
 =\dual v_0e^{s\dual\xi}+\dual q_0\,\mathrm D\exp_{s\dual\xi}[s\dual\eta].
\end{equation*}
Subtracting $d\dual Q/dt=\dual Q\dual\xi/T$ proves the formula.  At the left
endpoint, the Fr\'echet term vanishes because
$\mathrm D\exp_0[0]=0$ (although $\mathrm D\exp_0$ itself is the identity).
The right-endpoint reduction follows from exact algebraic interpolation.
\end{proof}

\begin{proposition}[Closed inverse of the exponential differential]
\label{prop:closed-dexp-inverse}
The derivatives $\dual v_i=\dual q_i^{(1)}$ are tangent vectors at
$\dual q_i$ and are not, in general, pure dual quaternions.  Purity applies
to their left- or right-trivialized representatives, such as
$\dual q_i^*\dual v_i$.  Consequently, the general expression for
$\dual{\delta a}$ in \eqref{eq:first-order-relative-data} must be retained.
In the numerical tests of Section~7, $\dual q_0=1$ and $\dual v_0$ is pure,
so only there one may write $\dual v_0^*=-\dual v_0$.

On the pure dual-quaternion subspace,
$\mathrm D\exp_{\dual\xi}$ is a linear isomorphism whenever
$0<\|\dual\xi_r\|<\pi$, where $\dual\xi_r$ is the primary
(real-quaternion) part of $\dual\xi$. On the six-dimensional pure
dual-quaternion space one recovers
$\dual\eta$ by solving the $8{\times}6$ linear system
\begin{equation}
 \bigl(\mathrm D\exp_{\dual\xi}\bigr)\dual\eta=\dual{\delta a}
 \label{eq:eta-linear-system}
\end{equation}
in least-squares form; the system is in fact consistent (residual at
machine precision) whenever $\dual{\delta a}$ is tangent at $\dual a$ and
$\dual\theta_r<\pi$, so ``least-squares'' refers to the solver used, not to a
rank deficiency of the map. Equivalently, on the dual-quaternion algebra the
inverse admits a closed form in quaternion products alone---no Lie
brackets, no infinite series.  Let
$\dual\theta=\|\dual\xi\|\in\mathbb R[\varepsilon_0]$ denote the dual
norm of the quaternion logarithm (one half of the physical rotation angle
in the real part), i.e.\ $\dual\theta_r=\|\dual\xi_r\|$ and
$\dual\theta_d=\langle\dual\xi_r,\dual\xi_d\rangle/\dual\theta_r$; scalar
functions of $\dual\theta$ are extended by the first-order Taylor rule
$f(\dual\theta_r+\varepsilon_0\dual\theta_d)=f(\dual\theta_r)+
\varepsilon_0\dual\theta_d f'(\dual\theta_r)$ inherent to the HMD calculus.
Then
\begin{equation}
 \dual\eta
  =\alpha(\dual\theta)\,\dual{\delta a}
  +\beta(\dual\theta)\bigl(\dual\xi\,\dual{\delta a}+\dual{\delta a}\,\dual\xi\bigr)
  +\gamma(\dual\theta)\,\dual\xi\,\dual{\delta a}\,\dual\xi,
 \label{eq:eta-rodrigues}
\end{equation}
with
\begin{equation}
 \alpha(\theta)=\frac{\theta}{2\sin\theta}+\frac{\cos\theta}{2},\qquad
 \beta(\theta)=-\frac{\sin\theta}{2\theta},\qquad
 \gamma(\theta)=\frac{\theta/\sin\theta-\cos\theta}{2\theta^2}.
 \label{eq:eta-coefficients}
\end{equation}
The identity $\alpha+\gamma\theta^2=\theta/\sin\theta$ links the diagonal and
rank-one parts.  Both \eqref{eq:eta-linear-system} and
\eqref{eq:eta-rodrigues} therefore give the unique pure solution.
\end{proposition}

\begin{proof}
It is enough first to work over ordinary quaternions; dualization then follows
coefficientwise from the first-order Taylor rule for analytic scalar
functions. Write $\xi=\theta u$, where $u^2=-1$, and decompose every pure
increment as
\begin{equation}
 \eta=\eta_{\parallel}+\eta_{\perp},\qquad
 \eta_{\parallel}=\tfrac12(\eta-u\eta u),\qquad
 \eta_{\perp}=\tfrac12(\eta+u\eta u).
 \label{eq:eta-parallel-perpendicular}
\end{equation}
Thus $\eta_{\parallel}$ is parallel to $u$, whereas
$u\eta_{\perp}+\eta_{\perp}u=0$. Differentiating
$\exp\xi=\cos\theta+u\sin\theta$ gives
\begin{equation}
 \mathrm D\exp_{\xi}[\eta_{\parallel}]
   =\exp(\xi)\eta_{\parallel},
 \qquad
 \mathrm D\exp_{\xi}[\eta_{\perp}]
   =\frac{\sin\theta}{\theta}\eta_{\perp}.
 \label{eq:dexp-parallel-perpendicular}
\end{equation}
Consequently the differential is invertible for $0<\theta<\pi$.
Substitute its two images separately into the right-hand side of
\eqref{eq:eta-rodrigues}. For a perpendicular increment the anticommutator
vanishes and $\xi\dual{\delta a}\xi=\theta^2\dual{\delta a}$; the multiplier
is one because $\alpha+\gamma\theta^2=\theta/\sin\theta$. For a parallel
increment, writing $\eta_{\parallel}=\lambda u$ and
$\dual{\delta a}=\lambda(-\sin\theta+u\cos\theta)$ reduces the scalar and
$u$ components respectively to
\begin{align*}
 -\alpha\sin\theta-2\beta\theta\cos\theta
       +\gamma\theta^2\sin\theta&=0,\\
 \alpha\cos\theta-2\beta\theta\sin\theta
       -\gamma\theta^2\cos\theta&=1,
\end{align*}
which follow directly from \eqref{eq:eta-coefficients}. Hence
\eqref{eq:eta-rodrigues} inverts $\mathrm D\exp_{\xi}$ on both invariant
subspaces and therefore on their direct sum. Replacing $\theta$ by the dual
norm and applying the dual Taylor extension proves the stated dual-quaternion
formula.
\end{proof}

\begin{remark}[Implementation and limiting cases]
The linear system \eqref{eq:eta-linear-system} and the closed expression
\eqref{eq:eta-rodrigues} agree to machine precision on the screw test of
\S7; the accompanying implementation uses \eqref{eq:eta-linear-system}.
The scalar functions $\alpha,\gamma$ are analytic on
$\theta_r\in(0,\pi)$ with removable singularities at $\theta_r=0$; at
$\theta_r\uparrow\pi$ they diverge as $1/\sin\theta_r$, while $\beta=-\sin\theta_r/(2\theta_r)$
remains bounded and vanishes there. The chart-switch
$\dual q\mapsto-\dual q$ resets the rotation angle inside
$(0,\pi)$ and admits the same closed-form.
For pure translations $\dual\xi_r=0$, the closed form
\eqref{eq:eta-rodrigues} remains valid but is not the identity.
The justification is the parity of the scalar functions: $\alpha,\beta,\gamma$
are even in $\theta$, so the Taylor extension
$f(\dual\theta_r+\varepsilon_0\dual\theta_d)=f(\dual\theta_r)+
\varepsilon_0\dual\theta_d f'(\dual\theta_r)$ evaluated at $\dual\theta_r=0$
reduces to the constant values
$\alpha(0)=1$, $\beta(0)=-1/2$, $\gamma(0)=1/3$ irrespective of any
directional limit of $\dual\theta_d$.  Because
$\dual\xi\,\dual{\delta a}\,\dual\xi$ contains a factor $\varepsilon_0^2=0$,
\eqref{eq:eta-rodrigues} reduces to
\begin{equation}
 \dual\eta=\dual{\delta a}-\tfrac12\bigl(\dual\xi\,\dual{\delta a}+\dual{\delta a}\,\dual\xi\bigr)
 \qquad\text{(pure translation, $\dual\xi_r=0$).}
 \label{eq:eta-pure-translation}
\end{equation}
The anticommutator term survives; $\mathrm D\exp_{\dual\xi}$ at
$\dual\xi=\varepsilon_0\dual\xi_d$ is not the identity on pure dual
quaternions, and identifying $\dual\eta$ with $\dual{\delta a}$ would
leave a first-order residue $\tfrac12\bigl(\dual\xi\,\dual{\delta a}+
\dual{\delta a}\,\dual\xi\bigr)$ on the dual part of
$\mathrm D\exp_{\dual\xi}[\dual\eta]-\dual{\delta a}$.
At $s=0$ and $s=1$, \eqref{eq:first-order-interior-defect} reduces to the
endpoint formulas \eqref{eq:left-endpoint-defect} and
\eqref{eq:right-endpoint-defect}.
\end{remark}

\begin{proposition}[Endpoint matching versus interior realizability]
If the inputs in \eqref{eq:transform-endpoint-sclerp} are differential
transforms, their encoded data are matched exactly at $s=0$ and $s=1$.
This endpoint statement does not imply that
$\mathcal C_k[\hhat Q](t)=0$ in the open time interval.
\end{proposition}

\begin{proof}
Endpoint matching is \eqref{eq:endpoints}, an equality in $\HU$.  Interior
realizability imposes the differential identities
\eqref{eq:coefficient-test}; these are not consequences of group membership or
of two endpoint equalities.
\end{proof}

\section{Holonomic HMD--Hermite--ScLERP}
\label{sec:hmd-hermite}

The generic obstruction above is removed by interpolating logarithmic
coordinates with a Hermite polynomial and only then applying the exponential.
This uses HMD automatic differentiation to obtain all endpoint coordinate jets
without differentiating $\mathrm{dexp}$.

Keep the initial pose $\dual q_0$ fixed as an ordinary dual quaternion and define the
two HMD logarithmic endpoint data
\begin{equation}
 \hbreve X_i=\Log(\dual q_0^*\hbreve q_i),
 \qquad i=0,1.
 \label{eq:hmd-log-data}
\end{equation}
They have coefficients in the dual-quaternion Lie algebra.  Write
\begin{equation}
 \dual X_i^{[k]}=k!\coef{k}\hbreve X_i,
 \qquad k=0,\ldots,n.
 \label{eq:log-jets}
\end{equation}
In particular, $\dual X_0^{[0]}=0$ and $\dual X_1^{[0]}=\dual\xi$.

Let $\dual H_{2n+1}(s)$ be the unique dual-Lie-algebra-valued Hermite polynomial of degree
at most $2n+1$ satisfying
\begin{equation}
 \frac{d^k\dual H_{2n+1}}{ds^k}(i)=T^k\dual X_i^{[k]},
 \qquad i=0,1,\quad k=0,\ldots,n.
 \label{eq:hermite-conditions}
\end{equation}
Existence and uniqueness follow componentwise from scalar Hermite
interpolation in any basis of the finite-dimensional dual-quaternion Lie
algebra; the resulting polynomial is basis independent.

\begin{definition}[Holonomic HMD--Hermite--ScLERP]
Define
\begin{equation}
 \dual Q_H(t)=\dual q_0\exp\dual H_{2n+1}(s(t)),
 \qquad
 \hbreve Q_H(t)=e^{\varepsilon D_t}\dual Q_H(t).
 \label{eq:hmd-hermite-sclerp}
\end{equation}
\end{definition}

\begin{theorem}[Bilateral jet interpolation at any prescribed finite order]
Assume all relative endpoint jets lie in the HMD prolongation of one admissible
logarithm chart and
$\hbreve q_0,\hbreve q_1$ are compatible order-$n$ differential transforms of
unit dual-quaternion curves.  Then \eqref{eq:hmd-hermite-sclerp} is holonomic
and satisfies
\begin{equation}
 \hbreve Q_H(t_0)=\hbreve q_0,
 \qquad
 \hbreve Q_H(t_1)=\hbreve q_1.
 \label{eq:hermite-jet-match}
\end{equation}
Thus it interpolates pose and all prescribed derivatives through order $n$ at
both endpoints.
\end{theorem}

\begin{proof}
By the functional calculus property of the differential transform,
\eqref{eq:hmd-log-data} is the order-$n$ differential transform of the local
coordinate curve $\dual X=\Log(\dual q_0^*\dual q)$.  Hence
\eqref{eq:log-jets} are precisely its endpoint derivatives.  The Hermite
conditions \eqref{eq:hermite-conditions} give equality of the coordinate jets.
Applying the analytic map $\dual X\mapsto \dual q_0\exp\dual X$ preserves equality of order-$n$
jets after the chain-rule scaling $d/dt=T^{-1}d/ds$.  Therefore
\eqref{eq:hermite-jet-match} holds.  Temporal holonomicity follows directly from
the definition $\hbreve Q_H=e^{\varepsilon D_t}\dual Q_H$.
\end{proof}

\begin{remark}[Compatibility of endpoint jets]
Compatibility means that each $\hbreve q_i$ is the temporal jet of a local
$C^n$ curve in $\DQ$.  Its relative coefficients therefore satisfy the
explicit relations \eqref{eq:relative-descriptor-constraints}--
\eqref{eq:relative-descriptor-recursion}.  Conversely, under the admissible
chart hypothesis in Proposition~\ref{prop:unit-hmd-jet-constraints}, a unit
HMD endpoint is locally realizable by a smooth unit dual-quaternion curve.
The unit relations concern each endpoint jet and should not be confused with
interior holonomicity of an interpolating HMD curve.
\end{remark}

\begin{proposition}[Left equivariance of the Hermite construction]
For a constant $\dual p\in\DQ$, simultaneous left multiplication of all endpoint
data by $\dual p$ produces $\dual{\widetilde Q}_H(t)=\dual p\dual Q_H(t)$.
\end{proposition}

\begin{proof}
The relative data are unchanged because
$(\dual p\dual q_0)^*(\dual p\hbreve q_i)=\dual q_0^*\hbreve q_i$.  Thus the Hermite polynomial is
unchanged and only its constant left factor becomes $\dual p\dual q_0$.
\end{proof}

\begin{remark}[Dependence on logarithmic coordinates]
Unlike the direct constant-generator ScLERP, the Hermite construction is not
generally invariant under a change of logarithm center, symmetric under
endpoint reversal, or geodesic.  These are the costs of matching arbitrary
bilateral jets in one selected coordinate chart.
\end{remark}

\begin{remark}[No explicit Lie-bracket recurrence]
The coefficients $\dual X_i^{[k]}$ are extracted from the finite HMD logarithms in
\eqref{eq:hmd-log-data}.  Noncommutative contributions are produced internally
by HMD arithmetic.  The construction therefore avoids explicit recursive
formulas for derivatives of $\mathrm{dexp}$, although such formulas can be used
for an independent equivalence proof.
\end{remark}

\subsection{Explicit first- and second-order formulas}

For $n=1$, expand the logarithmic endpoint transforms as
\begin{equation}
 \hbreve X_i=\dual X_i+\varepsilon\dual V_i,
 \qquad
 \dual X_i=\coef{0}\hbreve X_i,
 \quad \dual V_i=\coef{1}\hbreve X_i,
 \qquad i=0,1.
 \label{eq:first-order-log-data}
\end{equation}
Here $\dual X_0=0$ and $\dual X_1=\dual\xi$.  The cubic Hermite coordinate polynomial is
written with the normalized-time data
\begin{equation}
 \dual{\overline V}_i=T\dual V_i.
 \label{eq:scaled-first-order-data}
\end{equation}
It is
\begin{equation}
 \dual H_3(s)=h_{00}^{(3)}(s)\dual X_0+h_{10}^{(3)}(s)\dual{\overline V}_0
       +h_{01}^{(3)}(s)\dual X_1+h_{11}^{(3)}(s)\dual{\overline V}_1,
 \label{eq:cubic-hermite}
\end{equation}
where
\begin{align}
 h_{00}^{(3)}(s)&=2s^3-3s^2+1,
 &h_{10}^{(3)}(s)&=s^3-2s^2+s,\notag\\
 h_{01}^{(3)}(s)&=-2s^3+3s^2,
 &h_{11}^{(3)}(s)&=s^3-s^2.
 \label{eq:cubic-basis}
\end{align}
Consequently, the simplified first-order interpolant is
\begin{equation}
 \boxed{
 \dual H_3(s)=T(s^3-2s^2+s)\dual V_0+(-2s^3+3s^2)\dual\xi
       +T(s^3-s^2)\dual V_1,
 \qquad \dual Q_H(t)=\dual q_0\exp\dual H_3(s(t)).}
 \label{eq:first-order-explicit}
\end{equation}
It satisfies $\dual Q_H(t_0)=\dual q_0$, $\dual Q_H(t_1)=\dual q_1$ and matches both first-order
differential transforms.

For $n=2$, use the factorial convention
\begin{equation}
 \hbreve X_i=\dual X_i+\varepsilon\dual V_i+\frac{\varepsilon^2}{2}\dual A_i,
 \qquad
 \dual V_i=\coef{1}\hbreve X_i,
 \quad \dual A_i=2\coef{2}\hbreve X_i.
 \label{eq:second-order-log-data}
\end{equation}
Define the normalized-time acceleration data
\begin{equation}
 \dual{\overline V}_i=T\dual V_i,
 \qquad \dual{\overline A}_i=T^2\dual A_i.
 \label{eq:scaled-second-order-data}
\end{equation}
The quintic Hermite polynomial is
\begin{align}
 \dual H_5(s)={}&h_{00}^{(5)}(s)\dual X_0+h_{10}^{(5)}(s)\dual{\overline V}_0
           +h_{20}^{(5)}(s)\dual{\overline A}_0 \notag\\
          &+h_{01}^{(5)}(s)\dual X_1+h_{11}^{(5)}(s)\dual{\overline V}_1
           +h_{21}^{(5)}(s)\dual{\overline A}_1,
 \label{eq:quintic-hermite}
\end{align}
with basis functions
\begin{align}
 h_{00}^{(5)}(s)&=1-10s^3+15s^4-6s^5,
 &h_{01}^{(5)}(s)&=10s^3-15s^4+6s^5,\notag\\
 h_{10}^{(5)}(s)&=s-6s^3+8s^4-3s^5,
 &h_{11}^{(5)}(s)&=-4s^3+7s^4-3s^5,\notag\\
 h_{20}^{(5)}(s)&=\tfrac12(s^2-3s^3+3s^4-s^5),
 &h_{21}^{(5)}(s)&=\tfrac12(s^3-2s^4+s^5).
 \label{eq:quintic-basis}
\end{align}
Since $\dual X_0=0$ and $\dual X_1=\dual\xi$, the explicit second-order formula becomes
\begin{equation}
 \boxed{\begin{aligned}
 \dual H_5(s)={}&T(s-6s^3+8s^4-3s^5)\dual V_0\\
 &+\tfrac{T^2}{2}(s^2-3s^3+3s^4-s^5)\dual A_0\\
 &+(10s^3-15s^4+6s^5)\dual\xi\\
 &+T(-4s^3+7s^4-3s^5)\dual V_1\\
 &+\tfrac{T^2}{2}(s^3-2s^4+s^5)\dual A_1,
 \qquad \dual Q_H(t)=\dual q_0\exp\dual H_5(s(t)).
 \end{aligned}}
 \label{eq:second-order-explicit}
\end{equation}
Direct differentiation verifies
\begin{equation}
 \dual H_5(0)=0,\quad \dual H_5(1)=\dual\xi,\quad
 \dual H_5'(i)=T\dual V_i,\quad \dual H_5''(i)=T^2\dual A_i,
 \qquad i=0,1.
 \label{eq:quintic-checks}
\end{equation}
Because the endpoint quantities $\dual V_i$ and $\dual A_i$ are extracted from the HMD
logarithms rather than identified with physical twists and accelerations by
hand, all noncommutative conversion terms are already included.

\section{Recovery of higher-order acceleration fields}

For a genuine HMD differential transform $\hbreve q=e^{\varepsilon D_t}\dual q$, define
the HMD unit quaternion
\begin{equation}
 \hhat\varphi(t)=\hbreve q(t)\dual q(t)^*.
 \label{eq:field-descriptor}
\end{equation}
This notation reflects the two different objects in the product: $\hbreve q$
is the transformed motion, while $\dual q=\przero\hbreve q$ is its pose component.
The decomposition established in \cite{Condurache2025} is
\begin{equation}
 \hhat\varphi
 =\left(1+\frac{\varepsilon_0}{2}\widehat a\right)\widehat\phi,
 \qquad
\widehat a
 =2\frac{\partial\hhat\varphi}{\partial\varepsilon_0}\widehat\phi^*.
 \label{eq:unique-decomposition}
\end{equation}
The algebraic derivative with respect to the dual generator is the coefficient
extraction operation
\begin{equation}
 \frac{\partial}{\partial\varepsilon_0}
 (a+\varepsilon_0b)=b.
 \label{eq:dual-generator-derivative}
\end{equation}
For every body point represented by $r\in\mathbb R^3$, its HMD field is
\begin{equation}
 \widehat a_r
 =\widehat a+\Ad_{\widehat\phi}r
 =\widehat a+\widehat\phi r\widehat\phi^*.
 \label{eq:point-field}
\end{equation}
Equations \eqref{eq:field-descriptor}--\eqref{eq:point-field} identify, in one
object, the velocity, acceleration, jerk, and higher-order acceleration fields.
In particular,
\begin{equation}
 a_r^{[k]}=k!\,\coef{k}\widehat a_r,
 \qquad k=1,\ldots,n,
 \label{eq:coefficient-fields}
\end{equation}
under the coefficient convention of
\eqref{eq:differential-transform}.

No explicit $\mathrm{dexp}$ derivative or expanded Lie-bracket recurrence is
required to evaluate these quantities.  The noncommutative terms are retained
automatically by HMD quaternion multiplication.  Lie brackets may still be
used to prove equivalence with exponential-coordinate formulas, but they are
not needed in the HMD computational pipeline.

If the interpolant $\hhat Q(s(t))$ is proved temporally holonomic, then
\begin{equation}
 \hhat\varphi(t)=\hbreve Q(t)\dual Q(t)^*
 \label{eq:interpolated-fields}
\end{equation}
is the field descriptor of the interpolated rigid motion.  Without holonomicity,
$\hhat Q(s(t))\dual Q(t)^*$ remains a well-defined HMD quaternion but must not be
identified automatically with the acceleration fields of the time motion
$\dual Q(t)$.

\section{Reduction properties and numerical validation}

The algebraic construction has the expected reductions.  Removing the
multidual unit $\varepsilon$ gives ordinary ScLERP of unit dual quaternions;
removing also the dual unit $\varepsilon_0$ gives SLERP of unit quaternions.
Thus these constructions are related by reductions, not by set inclusions:
\begin{equation}
 \mathrm{HMD\!\text{-}ScLERP}
 \xrightarrow{\,\varepsilon=0\,}\mathrm{ScLERP}
 \xrightarrow{\,\varepsilon_0=0\,}\mathrm{SLERP}.
\end{equation}

For a reproducible noncommuting test, take the rotation-only subgroup,
$T=1$, $\dual q_0=1$, $\dual q_1=\exp(0.6\mathbf k)$, and prescribe
\begin{equation}
 \dual v_0=\mathbf i,\qquad \dual v_1=\dual q_1\mathbf j.
 \label{eq:numerical-data}
\end{equation}
These are valid tangent vectors at the two unit quaternions but are
incompatible with the constant generator $\dual\xi=0.6\mathbf k$.  Evaluation
of \eqref{eq:first-order-interior-defect} at
$s\in\{0,0.25,0.5,0.75,1\}$ yields the values in
Table~\ref{tab:first-order-contact-defect}: the endpoint defects reproduce
\eqref{eq:left-endpoint-defect}--\eqref{eq:right-endpoint-defect}, and the
interior defect is strictly positive because $\mathbf i$ and $\mathbf j$ do
not commute with $\mathbf k$.
\begin{table}[ht]
\centering
\caption{First-order contact defect for the noncommuting rotation-only test.
The norm is the Euclidean norm of the four real quaternion coefficients.}
\label{tab:first-order-contact-defect}
\begin{tabular}{c@{\qquad}ccccc}
\toprule
$s$ & $0$ & $0.25$ & $0.50$ & $0.75$ & $1$\\
\midrule
$\|\mathcal C_1\|$ & $1.166190$ & $0.890634$ & $0.773633$ & $0.890634$ & $1.166190$\\
\bottomrule
\end{tabular}
\end{table}
Thus exact endpoint interpolation in $\hhat U$ does not imply contact with
the pose projection in the interior.  The symmetry of the tabulated values about $s=1/2$ is a consequence
of the particular symmetric data chosen here and is not a general property of
the contact defect.  Using the same endpoint logarithmic jets in the cubic construction
\eqref{eq:first-order-explicit}, one verifies the substantive endpoint
property: the numerical time derivative
$d\dual Q_H/dt$ evaluated at $t=0$ reproduces $\dual v_0$ to $1.5\times10^{-14}$,
and at $t=T$ reproduces $\dual v_1$ to $7.9\times10^{-10}$
(central-difference error, script \texttt{verify\_hermite\_endpoints.py}).
By construction, $\mathcal C_1\equiv0$ throughout the interior because
$\dual Q_H(t)$ is itself a differential transform; this identity is not a
numerical test but a consequence of the Hermite lift.
The rotation-only example is a
subgroup test of the full dual-quaternion construction and is noncommutative
because the prescribed $\mathbf i$ and $\mathbf j$ directions do not commute
with the $\mathbf k$ generator.

We repeat the test on a proper screw motion in $SE(3)$.  With $T=1$, the
endpoints are $\dual q_0=1$ and
\begin{equation}
 \dual q_1=\cos(0.3)+\sin(0.3)\mathbf k
     +\varepsilon_0\bigl(-0.15\sin(0.3)+0.15\cos(0.3)\mathbf k\bigr),
\end{equation}
i.e.\ a screw with rotation angle $0.6$ about $\mathbf k$ combined with
translation $0.3$ along $\mathbf k$.  The prescribed endpoint jets are
\begin{equation}
 \dual v_0=\mathbf i+\varepsilon_0(0.2\mathbf i),\qquad
 \dual v_1=\dual q_1\bigl(\mathbf j+\varepsilon_0(0.1\mathbf k)\bigr),
 \label{eq:screw-numerical-data}
\end{equation}
valid tangent vectors at $\dual q_0$ and $\dual q_1$ that combine
noncommuting rotational directions with independent translational
velocities.  Solving
$\bigl(D\exp_{\dual\xi}\bigr)\dual\eta=\dual\delta a$ for the pure dual
quaternion $\dual\eta$ in \eqref{eq:first-order-interior-defect} yields the
values in Table~\ref{tab:screw-contact-defect}.  The endpoint defects reproduce
\eqref{eq:left-endpoint-defect}--\eqref{eq:right-endpoint-defect}, and the
tabulated values are no longer symmetric because $\dual v_0$ and $\dual v_1$
are not related by the discrete symmetry of the rotation-only test.

\begin{table}[ht]
\centering
\caption{First-order contact defect for the screw-motion test
\eqref{eq:screw-numerical-data}.  Norms are Euclidean norms on the eight
real coefficients of the dual quaternion.}
\label{tab:screw-contact-defect}
\begin{tabular}{c@{\qquad}cccccc}
\toprule
$s$ & $0$ & $0.2$ & $0.4$ & $0.6$ & $0.8$ & $1$\\
\midrule
$\|\mathcal C_1\|$ & $1.073546$ & $0.846771$ & $0.702285$ & $0.696668$ & $0.833194$ & $1.056894$\\
\bottomrule
\end{tabular}
\end{table}

Substituting the same endpoint dual-quaternion jets
\eqref{eq:screw-numerical-data} in \eqref{eq:first-order-explicit} produces
a pose curve $\dual Q_H(t)$ whose temporal prolongation is, by construction,
$e^{\varepsilon D_t}\dual Q_H(t)$; hence $\mathcal C_1\equiv0$ identically.
This statement is the holonomicity identity proved above, not a floating-point
estimate.  The interior of the direct HMD--ScLERP lift and its pose projection
therefore remain out of contact for screw motions as well, whereas the cubic
construction restores holonomicity in either case.  As a non-tautological check,
the script \texttt{verify\_hermite\_se3.py} compares the endpoint derivatives
of the resulting cubic pose curve with those of compatible local endpoint
curves in $SE(3)$.  The four pose-and-velocity residuals are respectively
$0$, $4.90\times10^{-14}$, $0$, and $4.20\times10^{-13}$.

\paragraph{Second-order (quintic) verification.}
For $n=2$ the same screw endpoints and logarithmic velocity data are augmented
with the pure dual-quaternion logarithmic accelerations
\begin{equation}
 \dual A_0=0.5\mathbf j+\varepsilon_0(0.1\mathbf i),\qquad
 \dual A_1=0.3\mathbf i+\varepsilon_0(0.1\mathbf j),
 \label{eq:screw-quintic-accel}
\end{equation}
with $T=1$.  To obtain compatible physical endpoint jets, define local
coordinate curves
\begin{equation}
 \dual q_i^{\rm loc}(\tau)=\dual q_0\exp\!\left(
 \dual X_i+\tau\dual V_i+\frac{\tau^2}{2}\dual A_i\right),
 \qquad i=0,1,
 \label{eq:local-compatible-jets}
\end{equation}
where $\dual V_i$ are the logarithmic velocities extracted from the first-order
data \eqref{eq:screw-numerical-data}.  The physical target derivatives are
defined by differentiating \eqref{eq:local-compatible-jets} at $\tau=0$.
The independently evaluated quintic curve
$\dual Q_H(t)=\dual q_0\exp\dual H_5(t)$ is then compared with these targets.
Five-point centered formulas with step $h=2\times10^{-4}$ are used for the
first and second derivatives;
the resulting residuals are shown in Table~\ref{tab:quintic-check}.
\begin{table}[ht]
\centering
\caption{Full quintic $SE(3)$ endpoint-jet check on the screw data
\eqref{eq:screw-numerical-data} and logarithmic accelerations
\eqref{eq:screw-quintic-accel}. Residuals are Euclidean norms on the eight
real coefficients of the dual quaternion.}
\label{tab:quintic-check}
\begin{tabular}{lc}
\toprule
Identity & Residual\\
\midrule
$\|\dual Q_H(0)-\dual q_0^{\rm loc}(0)\|$ & $0$\\
$\|\dot{\dual Q}_H(0)-\dot{\dual q}_0^{\rm loc}(0)\|$ & $5.11\times10^{-14}$\\
$\|\ddot{\dual Q}_H(0)-\ddot{\dual q}_0^{\rm loc}(0)\|$ & $2.31\times10^{-10}$\\
$\|\dual Q_H(1)-\dual q_1^{\rm loc}(0)\|$ & $0$\\
$\|\dot{\dual Q}_H(1)-\dot{\dual q}_1^{\rm loc}(0)\|$ & $5.17\times10^{-12}$\\
$\|\ddot{\dual Q}_H(1)-\ddot{\dual q}_1^{\rm loc}(0)\|$ & $2.67\times10^{-8}$\\
\bottomrule
\end{tabular}
\end{table}

\begin{figure}[ht]
\centering
\includegraphics[width=0.85\textwidth]{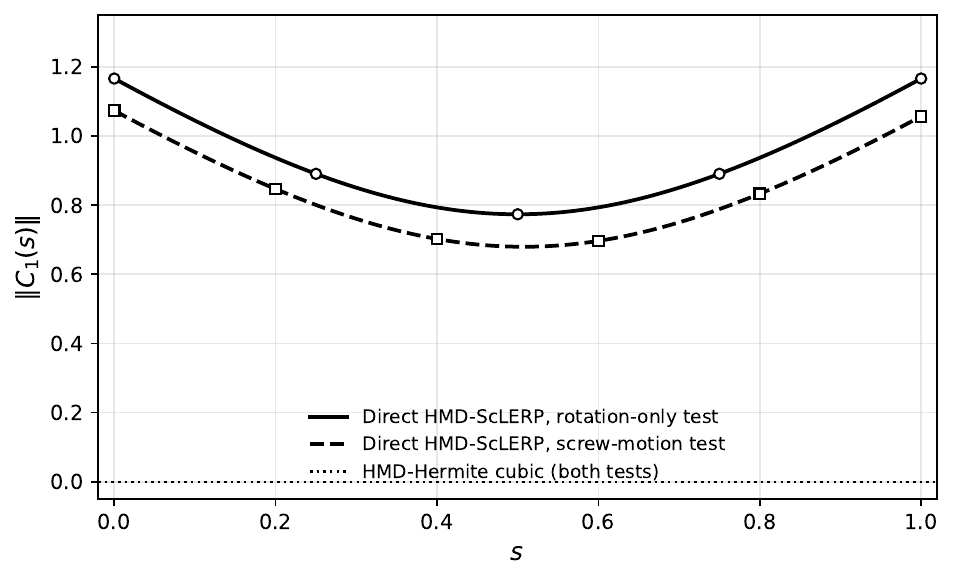}
\caption{First-order contact defect $\|\mathcal C_1(s)\|$ for the direct
HMD--ScLERP on the rotation-only test (solid) and the screw-motion test
(dashed).  Markers coincide with the sampled values of
Tables~\ref{tab:first-order-contact-defect}--\ref{tab:screw-contact-defect}.
The dotted line at $0$ is the profile of the cubic HMD--Hermite construction
\eqref{eq:first-order-explicit} in either case.}
\label{fig:contact-defect}
\end{figure}

\FloatBarrier

\begin{remark}[Multi-segment extension and cost]
The construction applies segmentwise to a knot sequence by matching the same
physical-time jets at every common node and selecting coherent logarithm
charts on adjacent segments.  Direct truncated Cauchy multiplication through
order $n$ costs $O(n^2)$ coefficient products; dualization changes the
constant factor, not this order.  The same truncated-product implementation
is used at every finite order $n$.
\end{remark}

\begin{remark}[Verification at higher orders]
Beyond the quintic ($n=2$) construction of
Table~\ref{tab:quintic-check}, the same Hermite algebra has been executed at
order $n=3$, producing the degree-$7$ septic basis.  A componentwise scalar
test of the eight endpoint identities
$H_7^{(m)}(i)=T^m\dual X_i^{[m]}$, $i\in\{0,1\}$, $m\in\{0,1,2,3\}$,
gives a maximum double-precision residual of $1.44\times10^{-12}$.  Since
these interpolation conditions act componentwise on the dual-quaternion
coefficients, the computation checks the polynomial implementation at $n=3$.
Extension to any prescribed finite order follows from the preceding theorem,
not from this numerical experiment.  Explicit septic coefficients and
residuals are provided in the accompanying implementation.
\end{remark}

\begin{remark}[Conditioning of the monomial basis at high order]
The condition number of the boundary matrix associated with the monomial
(Taylor) basis increases rapidly with $n$.  Numerically,
$\kappa\bigl(M_n\bigr)$ is $2.4\!\cdot\!10^{1}$, $7.6\!\cdot\!10^{2}$,
$4.8\!\cdot\!10^{4}$, $4.4\!\cdot\!10^{6}$, $5.5\!\cdot\!10^{8}$
for $n=1,\dots,5$ (script \texttt{verify\_higher\_order.py}).
Solving for the basis in double precision produces residuals
$0$, $10^{-14}$, $10^{-13}$, $10^{-11}$, and approximately
$1.9\times10^{-9}$ on the same range.  These values document the progressive
loss of numerical accuracy in the monomial basis; whether it remains adequate
at a given order depends on data scaling and the required tolerance.
Alternative polynomial bases may improve the conditioning, but their
quantitative assessment is outside the present validation.  Such a change of
basis would not alter the HMD Cauchy-product algebra.
\end{remark}

\section{Discussion}

ScLERP extends verbatim to the Lie group of unit HMD quaternions, but the direct
lift is generically non-holonomic.  Its endpoint coefficients prescribe
arbitrary velocities and accelerations, whereas its pose projection has one
constant generator $\dual\xi$.  The endpoint defects make this mismatch explicit.
The Hermite construction resolves it at the correct level: logarithmic endpoint
jets are obtained by HMD functional calculus, interpolated in the Lie algebra,
and lifted back to the group.  This separates the useful algebraic HMD--ScLERP
from the holonomic trajectory interpolant needed in mechanics.

\begin{remark}[Invariant tensorial prolongation]
\label{rem:park-ravani}
The invariant rotation splines of Park and Ravani \cite{ParkRavani1997}
admit a natural MD counterpart.  If $R_{\rm PR}(t)\in SO(3)$ is such a spline
and is $C^n$, then its temporal MD prolongation
\begin{equation}
 \breve R_{\rm PR}(t)
 =e^{\varepsilon D_t}R_{\rm PR}(t)
 =\sum_{k=0}^n\frac{\varepsilon^k}{k!}R_{\rm PR}^{(k)}(t)
 \label{eq:park-ravani-md-prolongation}
\end{equation}
is holonomic by construction, remains orthogonal by
Proposition~\ref{prop:hmd-orthogonal-preservation}, and satisfies the
equivariance
\begin{equation}
 e^{\varepsilon D_t}\bigl(P R_{\rm PR}(t)Q\bigr)
 =P\,\breve R_{\rm PR}(t)\,Q,
 \qquad P,Q\in SO(3)\text{ constant},
 \label{eq:park-ravani-md-equivariance}
\end{equation}
since $P$ and $Q$ are constant in time and $e^{\varepsilon D_t}$ acts only on
the temporal argument.  Thus the MD prolongation commutes with the fixed- and
moving-frame actions underlying the Park--Ravani bi-invariance, in the sense
that it is equivariant under constant left and right multiplications; this is
weaker than a full geometric bi-invariance of an interpolation scheme, which
is not claimed here.

In the dual orthogonal-tensor representation of rigid motion, the
coefficientwise argument extends via
Proposition~\ref{prop:hmd-orthogonal-preservation} and
identity~\eqref{eq:descriptor-commuting-identity} to an HMD transform
$\hbreve R=e^{\varepsilon D_t}\dual R$ with the holonomic descriptor
$\hhat\Psi=\hbreve R\dual R^T$ of \cite{Condurache2025}.  One should recall
that $SE(3)$ admits no bi-invariant metric, so on the dual level the
equivariance is naturally weaker than on $SO(3)$: it holds for constant left
and right actions, but not for a full bi-invariant geometric structure.  This
suggests a tensorial counterpart of the HMD--Hermite construction and a
representation-independence result through the HMD quaternion--tensor
covering.

The observation does not, however, license direct interpolation of arbitrary
HMD tensor knots.  Without compatibility with a common temporal source, an
interpolated HMD curve reproduces the obstruction of
Theorem~\ref{thm:endpoint-defects-non-holonomicity} at the tensor level.
Since the original Park--Ravani spline is $C^2$, its immediate temporal
prolongation supplies jets through order two.  For $n\ge3$, that immediate
prolongation of the original $C^2$ construction no longer suffices; the
HMD--Hermite construction of Section~\ref{sec:hmd-hermite} provides one
arbitrary-order alternative.  The two viewpoints are therefore complementary:
the Park--Ravani spline supplies a geometric base whose available temporal
prolongation is automatically holonomic, whereas HMD--Hermite directly matches
prescribed compatible endpoint jets at any finite order.
\end{remark}

\subsection{Applicability}
The HMD--Hermite--ScLERP construction targets problems in robot
programming and spacecraft guidance in which endpoint pose, twist, and
acceleration (and, at higher $n$, jerk and snap) are prescribed and the
interpolant must be a genuine rigid-motion history rather than a
group-valued curve of unrelated jets.  Typical instances are: online trajectory blending between
identified waypoints of an industrial manipulator, where endpoint
acceleration continuity is required for smooth torque profiles;
docking and rendezvous maneuvers with prescribed relative twist and
acceleration at contact; and animation or keyframing pipelines that
consume velocity-plus-acceleration data at each key.  In every case the input is a
pair of order-$n$ HMD dual quaternions, and the output is a $C^\infty$
curve $\dual Q_H(t)$ whose successive time derivatives recover the
prescribed data by finite algebra.

\subsection{Limitations}
The construction is local in three related senses.  First, the logarithm
chart of Definition~\ref{def:log-chart} excludes the rotation cut locus; if
the interpolated screw crosses $\theta_r=\pi$, the segment must be split
and a coherent branch selected on each subsegment.  Second, the Hermite
polynomial $\dual H_{2n+1}$ interpolates in Lie-algebra coordinates only,
so invariants such as constant screw axis or constant screw pitch are not
preserved in general: only screw-preserving inputs
(Corollary~\ref{cor:constant-screw}) produce a screw-invariant output.  Third, the arithmetic
cost of a single evaluation step is $\Theta(n^2)$ in the truncated Cauchy
product; for very large $n$ (well beyond typical mechanical orders of jerk
and snap), the memory footprint of the multidual coefficients dominates.
No claim of geometric bi-invariance beyond the constant left/right
equivariance of Remark~\ref{rem:park-ravani} is made. Finally, when the
endpoint jets are supplied by sampling a motion whose logarithmic coordinate
curve $\dual X$ belongs to $C^{2n+2}$ and has bounded derivative
$\dual X^{(2n+2)}$ on $[0,T]$, the standard bilateral Hermite remainder gives
an interior approximation error $O(T^{2n+2})$ in logarithmic coordinates.
A quantitative estimate tailored to the exponential lift is not proved here
and is left open.
Background on Hermite consistency and higher-order rigid-motion sampling
is in the classical treatments cited above.
A numerical cross-check against an independent $d\exp$/BCH implementation at
$n\geq 2$ would provide a useful implementation benchmark; such a comparison
is beyond the present scope.

\subsection{Extensions}
Three extensions are natural.  A multi-segment variant of
Section~\ref{sec:hmd-hermite} matches the same physical-time jets at every
internal node and produces a globally $C^n$ rigid-motion spline; the
construction is segmentwise, the chart selection is inherited from the
pose projection, and the resulting complexity is linear in the number of
segments.  A tensorial counterpart, obtained by pulling the Hermite
polynomial through the HMD quaternion--tensor covering, provides an
equivalent HMD--Hermite interpolation directly in $\dual{SO}_3$ and is
interesting when the ambient computational infrastructure is
tensor-based rather than quaternionic.  Finally, the same HMD apparatus
applies without change to constrained subgroups relevant to specific
mechanisms (e.g.\ Sch\"onflies motions $X(\mathbf a)\subset SE(3)$ for
pick-and-place manipulators),
by restricting the endpoint logarithms to the corresponding subalgebra.
These extensions preserve the central property emphasized throughout: the
interpolation is completed within the finite HMD product algebra, without
invoking iterated Lie brackets or Bernoulli series.

\section{Conclusion}

We introduced HMD--ScLERP as a group interpolation and proved that its direct
use with arbitrary endpoint transforms is generically non-holonomic.  The
distinction between $\hhat q$ and $\hbreve q$ prevents a generic HMD curve from
being mistaken for a differential transform.  We also characterized unit HMD
differential transforms by explicit binomial constraints at every order and
showed that the relative descriptor $\hbreve q\dual q^*$ is also unitary.  Its
coefficient identities determine recursively the conjugate-symmetric parts of
the higher-order relative fields from lower-order data.  A holonomic
HMD--Hermite--ScLERP was then obtained by Hermite interpolation of the HMD
logarithmic endpoint jets.  It matches bilateral data of arbitrary finite order
without explicit $\mathrm{dexp}$ differentiation.  For the resulting curve,
$\hbreve Q_H(t)\dual Q_H(t)^*$ yields all higher-order acceleration fields, with
noncommutative effects handled internally by HMD quaternion arithmetic.

\small\sloppy\raggedright

\fussy\normalsize


\begin{thebibliography}{99}

\bibitem{ParkRavani1997}
F. C. Park and B. Ravani, ``Smooth invariant interpolation of rotations,''
\emph{ACM Transactions on Graphics}, vol.~16, no.~3, pp.~277--295, 1997.
doi:10.1145/256157.256160.

\bibitem{Selig2005}
J. M. Selig, \emph{Geometric Fundamentals of Robotics}, 2nd ed.,
Springer, New York, 2005. doi:10.1007/b138859.

\bibitem{KavanEtAl2008}
L. Kavan, S. Collins, J. \v{Z}\'ara, and C. O'Sullivan,
``Geometric skinning with approximate dual quaternion blending,''
\emph{ACM Transactions on Graphics}, vol.~27, no.~4, article 105,
23 pages, 2008. doi:10.1145/1409625.1409627.

\bibitem{ZefranKumar1998}
M. \v{Z}efran and V. Kumar, ``Interpolation schemes for rigid body motions,''
\emph{Computer-Aided Design}, vol.~30, no.~3, pp.~179--189, 1998.
doi:10.1016/S0010-4485(97)00060-2.

\bibitem{Muller2014}
A. M\"uller, ``Higher derivatives of the kinematic mapping and some
applications,'' \emph{Mechanism and Machine Theory}, vol.~76, pp.~70--85,
2014. doi:10.1016/j.mechmachtheory.2014.01.007.

\bibitem{Muller2019}
A. M\"uller, ``An overview of formulae for the higher-order kinematics of
lower-pair chains with applications in robotics and mechanism theory,''
\emph{Mechanism and Machine Theory}, vol.~142, 103594, 2019.
doi:10.1016/j.mechmachtheory.2019.103594.

\bibitem{Messelmi2015}
F. Messelmi, ``Multidual numbers and their multidual functions,''
\emph{Electronic Journal of Mathematical Analysis and Applications},
vol.~3, no.~2, pp.~154--172, 2015.

\bibitem{FikeAlonso2012}
J. A. Fike and J. J. Alonso, ``Automatic differentiation through the use of
hyper-dual numbers for second derivatives,'' in \emph{Recent Advances in
Algorithmic Differentiation}, Springer, pp.~163--173, 2012.
doi:10.1007/978-3-642-30023-3\_15.

\bibitem{Condurache2025}
D. Condurache,
``An overview of higher-order kinematics of rigid body and multibody systems
with nilpotent algebra,''
\emph{Mechanism and Machine Theory}, vol.~209, 105959, 2025.
doi:10.1016/j.mechmachtheory.2025.105959.

\bibitem{KimKimShin1995}
M.-J. Kim, M.-S. Kim, and S. Y. Shin,
``A general construction scheme for unit quaternion curves with simple
high order derivatives,'' in \emph{Proc.\ SIGGRAPH '95}, ACM, pp.~369--376,
1995. doi:10.1145/218380.218486.

\bibitem{BeltaKumar2002}
C. Belta and V. Kumar,
``An SVD-based projection method for interpolation on $SE(3)$,''
\emph{IEEE Transactions on Robotics and Automation}, vol.~18, no.~3,
pp.~334--345, 2002. doi:10.1109/TRA.2002.1019463.

\bibitem{Zimmermann2020}
R. Zimmermann, ``Hermite interpolation and data processing errors on
Riemannian matrix manifolds,'' \emph{SIAM Journal on Scientific Computing},
vol.~42, no.~5, pp.~A2593--A2619, 2020.
doi:10.1137/19M1282878.

\bibitem{RodriguesEtAl2005}
R. C. Rodrigues, F. Silva Leite, and J. Jakubiak,
``A new geometric algorithm to generate smooth interpolating curves on
Riemannian manifolds,'' \emph{LMS Journal of Computation and Mathematics},
vol.~8, pp.~251--266, 2005. doi:10.1112/S146115700000098X.

\bibitem{Muller2025}
A. M\"uller,
``Geometric interpolation of rigid body motions,''
arXiv preprint arXiv:2509.16966, 2025.

\bibitem{Saunders1989}
D. J. Saunders, \emph{The Geometry of Jet Bundles}, Cambridge University
Press, Cambridge, 1989. doi:10.1017/CBO9780511526411.

\end{thebibliography}
\end{document}